\documentclass[letterpaper]{article}
\usepackage{uai2020}
\usepackage[margin=1in]{geometry}

\usepackage{times}

\usepackage{graphicx}

\usepackage{pgfplots}

\title{Complex Markov Logic Networks: Expressivity and Liftability}

\author{ 
{\bf Ond\v{r}ej Ku\v{z}elka} \\
Department of Computer Science \\
FEE CTU in Prague\\
Czech Republic
}

\usepackage{xcolor}
\usepackage{amsmath}
\usepackage{amsthm}
\usepackage{amssymb}
\usepackage{amsfonts}
\usepackage{dsfont}

\newtheorem{example}{Example}
\newtheorem{theorem}{Theorem}
\newtheorem{lemma}{Lemma}
\newtheorem{remark}{Remark}
\newtheorem{proposition}{Proposition}

\newtheorem{definition}{Definition}

\begin{document}

\maketitle

\begin{abstract}
We study expressivity of Markov logic networks (MLNs). We introduce {\em complex MLNs}, which use complex-valued weights, and we show that, unlike standard MLNs with real-valued weights, complex MLNs are {\em fully expressive}. We then observe that discrete Fourier transform can be computed using weighted first order model counting (WFOMC) with complex weights and use this observation to design an algorithm for computing {\em relational marginal polytopes} which needs substantially less calls to a WFOMC oracle than a recent algorithm.
\end{abstract}

\section{INTRODUCTION}

Statistical Relational Learning \cite{getoor2007introduction} (SRL) is concerned with learning probabilistic models from relational data such as, for instance, knowledge graphs, biological or social networks, structures of molecules etc. Markov Logic Networks \cite{Richardson2006} (MLNs) are among the most prominent SRL systems and in this paper we are interested in their {\em expressivity}.

Informally, {\em expressivity} measures the ``amount'' of distributions that can be modelled by a given class of probabilistic models. An MLN is given by a set of weighted first-order logic formulas and it defines a distribution on possible worlds over a given domain. Here we study expressivity of MLNs in a setting where we first fix the first-order logic formulas defining the MLN and then vary their weights. Since it is not even clear what {\em expressivity} should mean in this context, our first contribution in this paper is a formal framework for studying expressivity of MLNs.

The main reason for studying expressivity of MLNs in the setting where one first fixes the formulas is computational complexity of inference because its complexity usually depends mostly on the formulas and not so much on their weights.\footnote{Of course, the complexity of inference also depends on the length of the bit-representation of the weights.} This is studied in the area of SRL known as {\em lifted inference} \cite{DBLP:conf/ijcai/Poole03,braz2005lifted,van2011lifted,DBLP:conf/uai/GogateD11a}. Specifically, there are classes of MLNs for which certain inference problems can be performed in time polynomial in the size of the domain. Such classes are called {\em domain-liftable} and a prominent example are MLNs in which every first-order logic formula contains at most 2 logical variables~\cite{van2011lifted}. It is desirable to be able to represent as many distributions as possible using these restricted classes of MLNs. Motivated by the observation that most MLNs are not {\em fully expressive}, as our second contribution, we introduce {\em complex MLNs} ($\mathbb{C}$-MLNs), which use complex-valued weights, and show that they are fully expressive. This is in line with a recent work of Buchman and Poole \cite{DBLP:conf/uai/BuchmanP17} who introduced complex-valued weights into probabilistic logic programs in order to increase their expressivity (although the precise notion of expressivity used by them differs from our work).

Allowing complex-valued weights turns out to be useful for yet another reason. It allows us to compute discrete Fourier transform using weighted first order model counting. This, in turn, leads to our final contribution in this paper, which is an algorithm for computing {\em relational marginal polytopes} using a WFOMC oracle that needs substantially less oracle calls than a recent relational marginal polytope construction algorithm \cite{kuzelka.aistats.2020}.

\section{BACKGROUND}

In this section we provide the necessary background.


\subsection{DISCRETE FOURIER TRANSFORM}\label{sec:fourier}

Here we review the basic properties of {\em multi-dimensional Fourier transform} (DFT). Let $d$ be a positive integer and let $\mathbf{N} = [N_1, \dots, N_d] \in (\mathbf{N} \setminus \{ 0 \})^d$ be a vector of positive integers. Let us define $\mathcal{J} = \{0,1,\dots, N_1 - 1 \} \times \{0,1,\dots, N_2 - 1 \} \times \dots \times \{0,1,\dots, N_d - 1 \}$. Let $f : \mathcal{J} \rightarrow \mathbb{C}$ be a function defined on $\mathcal{J}$. Then the DFT of $f$ is the function $g : \mathcal{J} \rightarrow \mathbb{C}$ defined as
\begin{equation}\label{eq:dft}
    g(\mathbf{k}) = \sum_{\mathbf{n} \in \mathcal{J}} f(\mathbf{n})  e^{ - i 2 \pi \langle \mathbf{k}, \mathbf{n} / \mathbf{N} \rangle }
\end{equation}
where $\mathbf{k}/\mathbf{N} \stackrel{def}{=} \left[ [\mathbf{k}]_1/N_1, [\mathbf{k}]_2/N_2, \dots, [\mathbf{k}]_d/N_d \right]$ (i.e.\ ``/'' denotes component-wise division). We use the notation $g = \mathcal{F}\left\{f\right\}$. The inverse transform is then given as
\begin{equation}\label{eq:idft}
    f(\mathbf{n}) = \frac{1}{\prod_{l=1}^d N_l} \sum_{\mathbf{k} \in \mathcal{J}} g(\mathbf{k})  e^{ i 2 \pi \langle \mathbf{n}, \mathbf{k} / \mathbf{N} \rangle}.
\end{equation}
It holds $f = \mathcal{F}^{-1} \left\{ \mathcal{F}\left\{ f \right\} \right\}$.

We will need the DFT of the Kronecker delta function $\delta(\mathbf{n})$. Kronecker delta $\delta(\mathbf{n})$ is equal to $1$ when $\mathbf{n} = \mathbf{0}$ (here $\mathbf{0}$ is the zero vector) and $0$ otherwise. Its DFT has a conveniently simple form: $g(\mathbf{k}) = 1$. The DFT of the shifted Kronecker delta function $\delta(\mathbf{n}-\mathbf{n}_0)$ is $g(\mathbf{k}) = e^{ - i 2 \pi \langle \mathbf{k}, \mathbf{n}_0 / \mathbf{N} \rangle }$.


\subsection{FIRST ORDER LOGIC}

We assume a function-free first-order language defined by a set of constants $\Delta$, a set of variables $\mathcal{V}$ and a set of predicates (relations) $\mathcal{R}$. Variables start with lowercase letters and constants start with uppercase letters. An atom is $r(a_1,...,a_k)$ with $a_1,...,a_k\in \Delta \cup \mathcal{V}$ and $r\in \mathcal{R}$. A literal is an atom or its negation. For a first-order logic formula $\alpha$, we define $\textit{vars}(\alpha)$ to be the set of variables contained in it. 
A first-order logic formula in which none of the literals contains any variables is called {\em ground}. A possible world $\omega$ is represented as a set of ground atoms that are true in $\omega$. The satisfaction relation $\models$ is defined in the usual way: $\omega \models \alpha$ means that the formula $\alpha$ is true in $\omega$.

\subsection{MARKOV LOGIC NETWORKS}

A Markov logic network \cite{Richardson2006} (MLN) is a set of weighted first-order logic formulas $(\alpha,w)$, where $w\in \mathbb{R}$ and $\alpha$ is a function-free first-order logic formula. The semantics are defined w.r.t.\ the groundings of the first-order formulas, relative to some finite set of constants $\Delta$, called the domain. An MLN $\Phi$ induces the probability distribution over possible worlds $\omega \in \Omega$: 
\begin{equation}\label{eq:mln}
    P_{\Phi,\Omega}(\omega) = \frac{1}{Z} \exp \left(\sum_{(\alpha,w) \in \Phi} w \cdot N(\alpha,\omega)\right),
\end{equation}
where $N(\alpha, \omega)$ is the number of groundings of $\alpha$ satisfied in $\omega$ (when $\alpha$ does not contain any variables, we define $N(\alpha,\omega) = \mathds{1}(\omega \models \alpha)$), and $Z$, called {\em partition function}, is a normalization constant to ensure that $p_{\Phi}$ is a  probability distribution.


\subsection{WEIGHTED MODEL COUNTING}\label{sec:wfomc}

Computation of the partition function $Z$ of an MLN can be converted to a {\em first-order weighted model counting problem (WFOMC)}.

\begin{definition}[WFOMC \cite{van2011lifted}]
Let $w(P)$ and $\overline{w}(P)$ be functions from predicates to complex\footnote{Normally, in the literature, the weights of predicates are real numbers. However, we will need complex-valued weights in this paper, therefore we define the WFOMC problem accordingly using complex-valued weights.} numbers (we call $w$ and $\overline{w}$ {\em weight functions}) and let $\Gamma$ be a first-order theory. Then $\operatorname{WFOMC}(\Gamma,w,\overline{w}) =$
$$
     = \sum_{\omega \in \Omega : \omega \models \Gamma} \prod_{a \in \mathcal{P}(\omega)} w(\textit{Pred}(a)) \prod_{a \in \mathcal{N}(\omega)} \overline{w}(\textit{Pred}(a))
$$
where $\mathcal{P}(\omega)$ and $\mathcal{N}(\omega)$ denote the positive literals that are true and false in $\omega$, respectively, and $\textit{Pred}(a)$ denotes the predicate of $a$ (e.g. $\textit{Pred}(\textit{friends}(\textit{Alice},\textit{Bob})) = \textit{friends}$).
\end{definition}

To compute the partition function $Z$ using weighted model counting, we may proceed as in \cite{van2011lifted}. Let an MLN $\Phi = \{(\alpha_1,w_1),\dots,(\alpha_m,w_m) \}$ be given. 
For every weighted formula $(\alpha_i,w_i) \in \Phi$, where the free variables in $\alpha_i$ are exactly $x_1$, $\dots$, $x_k$, we create a new formula 
$$
    \forall x_1,\dots,x_k : \xi_i(x_1,\dots,x_k) \Leftrightarrow \alpha_i(x_1,\dots,x_k)
$$
where $\xi$ is a new fresh predicate. We denote the resulting set of new formulas $\Gamma$. Then we set
$w(\xi_i) = \exp{\left(w_i \right)}$
and $\overline{w}(\xi_i) = 1$ and for all other predicates we set both $w$ and $\overline{w}$ equal to 1. It is easy to check that then $\mathbf{WFOMC}(\Gamma,w,\overline{w}) = Z$, which is what we needed to compute. To compute the marginal probability of a given query $q$, we have $\textit{Pr}_{\Phi,\Omega}[q] = \frac{\mathbf{WFOMC}(\Gamma \cup \{ q \}, w, \overline{w})}{\mathbf{WFOMC}(\Gamma, w, \overline{w})}$.



\subsection{DOMAIN-LIFTED INFERENCE}

Importantly, there are classes of first-order logic theories for which weighted model counting is polynomial-time. In particular, as shown in \cite{van2014skolemization}, when the theory $\Gamma$ consists only of first-order logic sentences, each of which contains at most two logic variables, the weighted model count can be computed in time polynomial in the number of elements in the domain $\Delta$ over which the set of possible worlds $\Omega$ is defined. It follows from the translation described in the previous section that this also means that computing the partition function of $2$-variable MLNs can be done in time polynomial in the size of the domain. This is not the case in general when the number of variables in the formulas is greater than two unless P = \#P$_1$~\cite{beame2015symmetric}.\footnote{\#P$_1$ is the set of \#P problems over a unary alphabet.} Within statistical relational learning, the term used for problems that have such polynomial-time algorithms is {\em domain liftability}.

\begin{definition}[Domain liftability]
An algorithm for computing WFOMC with real weights is said to be domain-liftable if it runs in time polynomial in the size of the domain.
\end{definition}

In this work we will also need {\it domain liftability over $\mathbb{C}$} which differs from the classical definition by allowing complex-valued weight functions $w(.)$ and $\overline{w}(.)$.

\begin{definition}[Domain liftability over $\mathbb{C}$]
An algorithm for computing WFOMC with complex-valued weight functions $w(.)$ and $\overline{w}(.)$ is said to be domain-liftable over $\mathbb{C}$ if it runs in time polynomial in the size of the domain.
\end{definition}

One can show, by inspecting the respective domain-lifted algorithms from the literature (e.g.\ \cite{van2011lifted,van2014skolemization,beame2015symmetric}) that these algorithms can be modified to yield domain-lifted algorithms over $\mathbb{C}$ (we discuss this in a bit more detail in Section \ref{sec:numerics}). 


\section{EXPRESSIVITY OF MLNS}

In this section we lay down the framework that we need in order to be able to talk about expressivity of MLNs.

\subsection{A MOTIVATING EXAMPLE}\label{sec:motivating_example}

\begin{figure}[t]
\begin{tikzpicture}
[
    declare function={binom(\k,\n,\p)=\n!/(\k!*(\n-\k)!)*\p^\k*(1-\p)^(\n-\k);},
    declare function={binomm(\kk,\nn,\pp)=((-1)^\kk + 1)*(\nn!/(\kk!*(\nn-\kk)!)*\pp^\kk*(1-\pp)^(\nn-\kk));}
]
\begin{axis}[
    samples at={0,...,60},
    yticklabel style={
        /pgf/number format/fixed,
        /pgf/number format/fixed zerofill,
        /pgf/number format/precision=1
    },
    legend pos=north west
]
\addplot [only marks, cyan] {binom(x,60,exp(-1)/(exp(-1)+1))}; \addlegendentry{$w=-1$}
\addplot [only marks, orange] {binom(x,60,exp(0)/(exp(0)+1))}; \addlegendentry{$w=0$}
\addplot [only marks, blue] {binom(x,60,exp(1)/(exp(1)+1))}; \addlegendentry{$w=1$}
\addplot [only marks, black] {binomm(x,60,exp(0)/(exp(0)+1))}; \addlegendentry{$\mathbf{w}=[0,\pi \cdot i]$}
\end{axis}
\end{tikzpicture}
\caption{Illustration of count distributions induced by three MLNs $\Phi_1 = \{ (\textit{heads}(x), -1) \}$ (cyan), $\Phi_2 = \{ (\textit{heads}(x), 0) \}$ (orange) and $\Phi_3 = \{ (\textit{heads}(x), 1) \}$ (blue) and a {\em complex} MLN $\Phi_\mathbb{C} = \{ (\textit{heads}(x), [0, \pi \cdot i]) \}$ (black) over a domain of size 60.}\label{fig1}
\end{figure}
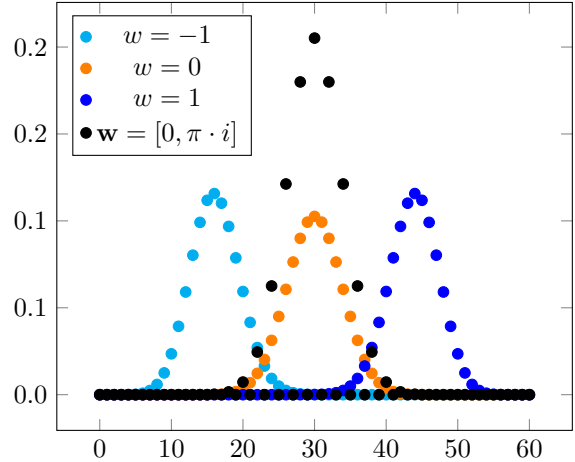

Let us consider an MLN $\Phi$, consisting of a single formula $\alpha = \textit{heads}(x)$ with weight $w$ on a domain $\Delta$. Which distributions can $\Phi$ model? To answer this question, let $O$ be the random variable, taking values in the set of possible worlds $\Omega$, sampled from the distribution given by $\Phi$. 
It turns out that this $\Phi$ can only model distributions for which the random variable $N(\alpha,O)$ is distributed as a binomial random variable, which is quite restrictive.

There are certainly limits as to which distributions we can reasonably expect to be able to represent with the MLN $\Phi$. On the one hand, we cannot expect the representable distributions to allow us to assign different probabilities to two possible worlds $\omega$ and $\omega'$ such that $N(\alpha,\omega) = N(\alpha,\omega')$. On the other hand, that does not yet mean that $N(\alpha,O)$ should be distributed as a binomial random variable. For instance, we might want $N(\alpha,O)$ to be distributed uniformly over the interval $[0;|\Delta|]$. 
We will show in this paper that this is indeed possible and that we can represent any such distribution using MLNs if we allow weights of formulas to be complex numbers (more precisely vectors of complex numbers). To illustrate this, in Figure \ref{fig1}, we show count distributions of three MLNs: $\Phi_1 = \{ (\textit{heads}(x), -1) \}$, $\Phi_2 = \{ (\textit{heads}(x), 0) \}$ and $\Phi_3 = \{ (\textit{heads}(x), 1) \}$ and a {\em complex} MLN $\Phi_\mathbb{C} = \{ (\textit{heads}(x), [0, \pi \cdot i]) \}$. Although we have not yet introduced complex MLNs formally, this example gives heads up for expressivity of complex MLNs, as the count distribution shown in Figure \ref{fig1} for $\Phi_\mathbb{C}$ is clearly not a binomial distribution (notice that it is zero for all odd numbers).

\subsection{MEASURING EXPRESSIVITY: SETUP}

We need to set up a language that we will use in this paper to talk about expressivity of MLNs. Given an MLN $\Phi = \{(\alpha_1,w_1), \dots, (\alpha_m,w_m) \}$ and a domain $\Delta$, we introduce the vectors of the count-statistics: 
$$\mathbf{N}(\Phi,\omega) \stackrel{def}{=} [N(\alpha_1,\omega),\dots,N(\alpha_m,\omega)].$$ 
We are interested in the distribution of the random vector-valued variable $\mathbf{N}(\Phi,O)$ where $O$ is sampled from the MLN $\Phi$. We call this distribution {\it count distribution} of the MLN $\Phi$.

\begin{definition}[Count Distribution]
Let $\Phi = \{(\alpha_1,w_1), \dots, (\alpha_m,w_m) \}$ be an MLN defining a distribution over a set of possible worlds $\Omega$. The count distribution of $\Phi$ is the distribution of $d$-dimensional vectors of non-negative integers $\mathbf{n}$ given by
$$q_\Phi(\mathbf{n},\Omega) =  \sum_{\omega \in \Omega : \mathbf{N}(\Phi,\omega) = \mathbf{n}} p_{\Phi,\Omega}(\omega)$$
where $p_{\Phi,\Omega}$ is the distribution given by the MLN $\Phi$.
\end{definition}

\begin{remark}
If $\mathbf{N}(\Phi,\omega) = \mathbf{N}(\Phi,\omega')$ then necessarily $P_{\Phi,\Omega}(\omega) = P_{\Phi,\Omega}(\omega')$. It follows that we do not lose any information by focusing on the respective count distributions. 
\end{remark}

Another important concept, which we need, is the {\it support} of a set of formula.

\begin{definition}[Support]
Let $\Omega$ be a set of possible worlds and $\Psi = \{\alpha_1,\dots,\alpha_m\}$ a set of first-order logic formulas. 
We define the support of $\Psi$ on $\Omega$ to be the set
$$\operatorname{Supp}(\Psi,\Omega) = \{ \mathbf{N}(\Psi,\omega) | \omega \in \Omega \}.$$
\end{definition}

After rescaling, the convex hull of $\operatorname{Supp}(\Psi,\Omega)$ is equal to the {\em relational marginal polytope} of $\Psi$ \cite{kuzelka2018relational}. 

\subsection{FULL EXPRESSIVITY}\label{sec:full_expr_def}

Now we can finally describe in detail what we will mean by full expressivity of MLNs. 

\begin{definition}[(Almost) Full Expressivity]
Let $\Omega$ be a set of possible worlds and $\Psi = \{\alpha_1,\dots,\alpha_m \}$ be a set of first-order logic formulas. We say that a class of MLNs given by $\Psi$ is (almost\footnote{Here, the term {\em almost} is used in measure-theoretic sense.}) fully expressive if the following holds: For (almost) any distribution $Q$ on $\operatorname{Supp}(\Psi,\Omega)$ there exists an MLN $\Phi$ such that its count distribution $q_{\Phi,\Omega}$ is equal to $Q$.
\end{definition}

\noindent It follows from the example in Section \ref{sec:motivating_example} that in general, when restricted to real-valued weights, MLNs are not fully expressive. However, as we show later in this paper, {\em complex MLNs} will turn out to be fully expressive.

We end this section with the following negative result. 

\begin{proposition}\label{prop:prop1}
Let $\Psi = \{\alpha_1,\dots,\alpha_m \}$ be a set of first-order logic formulas and $\Omega$ be a set of possible worlds and define $D = |\operatorname{Supp}(\Psi,\Omega)|$. Let $\mathbb{P}^{(D)}$ denote the $D$-dimensional probability simplex, i.e.\ $\mathbb{P}^{(D)} = \{(p_1,\dots, p_D) \in [0;1]^D | \sum_{i=1}^D p_i = 1 \}$. If $D > m+1$ then the set of count distributions representable by MLNs $\Phi$ of the form $\Phi = \{(\alpha_1,w_1),\dots,(\alpha_m,w_m) \}$, where $[w_1,\dots,w_m] \in \mathbb{R}^m$, has measure-zero as a subset of~$\mathbb{P}^{(D)}$.
\end{proposition}
\begin{proof}
(Sketch) Let $f : \mathbb{R}^m \rightarrow \mathbb{P}^{(D)}$ denote the map from the weights of the MLN to the probabilities of the count vectors given by the count distribution $q_{\Phi,\Omega}(\mathbf{n})$ (here we can assume that the vectors from $\operatorname{Supp}(\Psi,\Omega)$ are ordered lexicographically, hence the $i$-th element of the vector $f(\mathbf{w})$ corresponds to $q_{\Phi,\Omega}(\mathbf{n}_i)$ where $\mathbf{n}_i$ is the $i$-th vector according to this ordering). It follows from the definition of MLNs, in particular from (\ref{eq:mln}), and from the definition of the count distribution that the map $f$ is continuously differentiable. Hence, we can apply Sard's theorem \cite{sard1942measure} and conclude that $f(\mathbb{R}^m)$ has measure zero in~$\mathbb{P}^{(D)}$.\footnote{Using Sard's theorem may be an overkill for our simple application but it gives us the result we need in a relatively straightforward way.}
\end{proof}

Intuitively, the result above is saying nothing more than that, in most cases, we do not have enough parameters to represent any count distribution. In practice, the representable distributions may often be enough (but not always!).

\subsection{WHEN ARE MLNS FULLY EXPRESSIVE?}\label{sec:when_are_mlns_fully_expressive}

Sometimes MLNs are (almost) fully expressive and a natural question to ask is when this is the case. Proposition \ref{prop:prop1} provides us with a necessary condition: the number of first-order logic formulas in the MLN must be at least equal to the size of the set $\operatorname{Supp}(\Psi,\Omega)$ minus 1.

We first take a look at an example. This example is about the case when we have one formula for every possible world $\omega$ that describes the world completely. In particular, we assume that for every possible world $\omega$, the MLN contains a conjunction $\alpha_\omega$ of all ground literals (positive and negative) over the domain $\Delta$ which are true in $\omega$. One can show that such MLNs are fully expressive, as we illustrate in the next example.

\begin{example}
Let $\Delta = \{\textit{Alice},\textit{Bob}\}$ and let us assume there is only one unary relation $\textit{sm}/1$ and no higher-arity relations. We consider the following MLN
\begin{align*}
    \Phi =& \{ (\textit{sm}(\textit{Alice}) \wedge \textit{sm}(\textit{Bob}), w_1), \\
    &(\textit{sm}(\neg \textit{Alice}) \wedge \textit{sm}(\textit{Bob}), w_2), \\
    & (\textit{sm}(\textit{Alice}) \wedge \neg \textit{sm}(\textit{Bob}), w_3), \\
    & (\neg \textit{sm}(\textit{Alice}) \wedge \neg \textit{sm}(\textit{Bob}), w_4)\}.
\end{align*}
It is not difficult to see that $\Phi$ is almost fully expressive if we do not allow infinite weights and fully expressive if infinite weights are allowed. Indeed, for any given $q_1,q_2,q_3,q_4 \in [0;1]$ such that $q_1+q_2+q_3+q_4=1$, it is enough to set $w_i = \ln\frac{q_i}{1-q_i}$.
\end{example}

Generalizing the reasoning from the above example, we have a sufficient condition for an MLN to be fully expressive: the formulas defining it must satisfy that, for any $\omega \in \Omega$, $\omega \models \alpha_i$ holds for exactly one formula $\alpha_i$ from the MLN. As a sanity check, one can also see that if this is the case, the necessary condition is trivially satisfied as well.

We note that the sufficient condition identified in this section is usually not satisfied in practice, though.\footnote{Existing implementations of MLNs support a syntactic sugar ``+'' for grounding selected logical variables. Such MLNs can also be studied in our framework. However, using ``+'' may often lead to more intractable inference, as lifted inference algorithms are polynomial in the domain-size but not in the number of formulas, and, on its own, it still does not guarantee full expressivity even in the simplest case illustrated in Section~\ref{sec:motivating_example} (note that, in particular, the MLN $\Phi=\{(\textit{heads}(+x),w)\}$ does not satisfy the sufficient condition from this section).} Hence we need a different approach to obtain full expressivity. In this paper we add complex weights to MLNs for this purpose.

\section{$\mathbb{C}$-MLNS}\label{sec:complex}

In this section we introduce MLNs with complex weights, which we call {\em complex} MLNs ($\mathbb{C}$-MLNs). As it turns out, just replacing real weights by complex weights would not bring us much expressivity. We need to allow having a vector of complex weights for every formula in the MLN. 

We start by defining complex MLNs formally.

\begin{definition}[$\mathbb{C}$-MLNs]
Let $\Omega$ be a set of possible worlds. A complex MLN is a set $\Phi = \{ (\alpha_1, \mathbf{w}_1), \dots, (\alpha_m,\mathbf{w}_m) \}$ where $\alpha_i$'s are first-order logic formulas and $\mathbf{w}_i$'s are vectors of complex numbers of the same dimension $d$, i.e., $\mathbf{w}_1,\dots,\mathbf{w}_m \in \mathbb{C}^d$. We define the probability given by $\Phi$ as
$$p_{\Phi,\Omega}(\omega) = \frac{1}{Z} \sum_{i=1}^d \exp{\left( \sum_{j=1}^m [\mathbf{w}_j]_i \cdot N(\alpha_j,\omega) \right)},$$
where $[\mathbf{w}_j]_i$ denotes the $i$-th entry of the vector $\mathbf{w}_j$ and
$$Z = \sum_{\omega \in \Omega} \sum_{i=1}^d \exp{\left( \sum_{j=1}^m [\mathbf{w}_j]_i \cdot N(\alpha_j,\omega) \right)}.$$
A complex MLN $\Phi$ is called {\em proper} if $p_{\Phi,\Omega}(\omega) \in [0;1]$ for all $\omega \in \Omega$.
\end{definition}

\noindent In what follows we will only work with valid MLNs; we will omit the term ``proper'' when there is no risk of confusion.

There are two main differences w.r.t.\ normal MLNs. The first (and obvious) one is that we allow complex weights. The second is that the expression that defines probability of a possible world is a sum of exponentials in the case of complex MLNs as opposed to a single exponential in the case of normal MLNs. This is needed to guarantee full expressivity (following the reasoning in the proof of Proposition \ref{prop:prop1} in Section \ref{sec:full_expr_def} and the discussion therein). 

The next example shows that we can express more distributions using complex MLNs than using normal MLNs (later we will show that complex MLNs are actually fully expressive but here we want to just give an illustration).

\begin{example}
Let $\alpha = \textit{heads}(x)$ be a first-order logic formula and $\mathbf{w} = [0, \pi \cdot i]$ where $i$ is the imaginary unit. Let $\Phi = \{ (\alpha, \mathbf{w}) \}$ be a complex MLN. Let $\Delta = \{A,B,C,D\}$ be the domain of the MLN and $\Omega = 2^{\{\textit{heads}(A), \textit{heads}(B), \textit{heads}(C), \textit{heads}(D)\}}$ the respective set of all possible worlds. We will now compute the distribution of $\mathbf{N}(\Phi,O)$ where $O$ is sampled from the distribution given by the complex MLN $\Phi$ over the domain $\Omega$. First we compute the partition function. We have: 
$Z = 16$
which can be computed by brute-force enumeration. Then we obtain:
\begin{align*}
    P[N(\alpha,O) = 0] &= \left(\begin{array}{c} 4 \\ 0 \end{array}\right) \cdot \frac{e^{0} + e^{0}}{16} &= \frac{1}{8}, \\
    P[N(\alpha,O) = 1] &= \left(\begin{array}{c} 4 \\ 1 \end{array}\right) \cdot \frac{e^{0} + e^{\pi \cdot i}}{16} &= 0, \\
    P[N(\alpha,O) = 2] &= \left(\begin{array}{c} 4 \\ 2 \end{array}\right) \cdot \frac{e^{0} + e^{2 \pi \cdot i}}{16} &= \frac{3}{4}, \\
    P[N(\alpha,O) = 3] &= \left(\begin{array}{c} 4 \\ 3 \end{array}\right) \cdot \frac{e^{0} + e^{3 \pi \cdot i}}{16} &= 0, \\
    P[N(\alpha,O) = 4] &= \left(\begin{array}{c} 4 \\ 4 \end{array}\right) \cdot \frac{e^{0} + e^{4 \pi \cdot i}}{16} &= \frac{1}{8}.
\end{align*}
The distribution $N(\alpha,O)$ is obviously not a binomial distribution. Hence, this is already an example of a distribution that could not be encoded by an MLN of the form $\Phi = \{ (\textit{heads}(x), w) \}$ but that can be represented by the respective complex MLN.
\end{example}

In Section \ref{sec:inference} we show that inference in complex MLNs can be performed using WFOMC in a way completely analogical to the classical case.

\section{$\mathbb{C}$-MLNS ARE FULLY EXPRESSIVE}

In this section we show that $\mathbb{C}$-MLNs are actually fully expressive (if they contain the trivial formula $\top$, i.e.\ tautology). To show this we first show how to obtain a $\mathbb{C}$-MLN whose count-distribution is equal to the Kronecker $\delta$-function $\delta(\mathbf{n}-\mathbf{n}_0)$.

\begin{lemma}\label{lemma:kronecker}
Let 
$\Phi = \{(\alpha_1, \mathbf{w}_1),$ $\dots,$ $(\alpha_m, \mathbf{w}_m),$ $(\top, \mathbf{w}_\top) \}$ 
be a $\mathbb{C}$-MLN and $\Delta$ be the set of domain elements and $\Omega$ the set of all possible worlds on this domain. Let us define 
\begin{multline*}
    \mathcal{J} = \left\{0,1,2,\dots, |\Delta|^{|\textit{vars}(\alpha_1)|} \right\} \times \dots \\
    \times \left\{0,1,2,\dots, |\Delta|^{|\textit{vars}(\alpha_m)|} \right\},
\end{multline*}
$$\mathbf{M} = [|\Delta|^{|\textit{vars}(\alpha_1)|}+1, \dots, |\Delta|^{|\textit{vars}(\alpha_m)|}+1, 1],$$
and 
$$\mathbf{W}(\mathbf{k}) =  i 2\pi \cdot \left( \mathbf{k}/\mathbf{M} - [0,\dots,0, \langle \mathbf{k}/\mathbf{M}, \mathbf{n}_0 \rangle ]\right),$$
where ``/'' denotes the component-wise division of the two vectors and $\mathbf{n}_0 \in \mathcal{J} \times \{1\}$. Let us order the elements of $\mathcal{J} \times \{ 0 \}$ arbitrarily and denote $\mathbf{j}(j)$ the $j$-th element of $\mathcal{J}$.
Using the above notations, let us define the weights of the $\mathbb{C}$-MLN $\Phi$ as follows:
\begin{align*}
    \mathbf{w}_1 &= \left[ [\mathbf{W}(\mathbf{j}(1))]_1, [\mathbf{W}(\mathbf{j}(2))]_1, \dots, [\mathbf{W}(\mathbf{j}(|\mathcal{J}|))]_1 \right]\\
    & \dots \\
    \mathbf{w}_m &= \left[ [\mathbf{W}(\mathbf{j}(1))]_m, [\mathbf{W}(\mathbf{j}(2))]_m, \dots, [\mathbf{W}(\mathbf{j}(|\mathcal{J}|))]_m \right]\\
    \mathbf{w}_\top &= \left[ [\mathbf{W}(\mathbf{j}(1))]_{m+1}, \dots, [\mathbf{W}(\mathbf{j}(|\mathcal{J}|))]_{m+1} \right].
\end{align*}
Then 
$$p_{\Phi,\Omega}(\omega) = \begin{cases} \frac{1}{Z} & \mbox{if } \mathbf{N}(\Phi,\omega) = \mathbf{n}_0 \\ 0 & \mbox{otherwise} \end{cases}.$$
\end{lemma}
\begin{proof}
First we rewrite the probability of a possible world $\omega$ induced by the $\mathbb{C}$-MLN $\Phi$ in a more compact form using the scalar product notation $\langle,\rangle$:
\begin{multline*}
p_{\Phi,\Omega}(\omega) = \frac{1}{Z} \sum_{\mathbf{j} \in \mathcal{J} \times \{ 0 \}} e^{\left( \langle \mathbf{W}(\mathbf{j}), \mathbf{N}(\Phi,\omega) \rangle \right)} \\
= \frac{1}{Z} \sum_{\mathbf{j} \in \mathcal{J} \times \{0\}} e^{\left( \langle i 2\pi \mathbf{j}/\mathbf{M}, \mathbf{N}(\Phi,\omega) \rangle - i 2 \pi \langle \mathbf{j}/\mathbf{M}, \mathbf{n}_0 \rangle \right)}.
\end{multline*}
We can notice (cf Section \ref{sec:fourier}) that this is nothing else than $\frac{1}{Z} \delta({\mathbf{N}(\Phi,\omega)-\mathbf{n}_0})$, which finishes the proof of this lemma.
\end{proof}

\begin{theorem}\label{thm:expressivity}
Any $\mathbb{C}$-MLN containing the formula $\top$ is fully expressive.
\end{theorem}
\begin{proof}
Let $\Phi$ and $\mathcal{J}$ be as in Lemma \ref{lemma:kronecker}. We proceed as follows. For every $\mathbf{j} \in \mathcal{J} \times \{1\}$ we construct a $\mathbb{C}$-MLN $\Phi_{\mathbf{j}}$ such that $p_{\Phi_{\mathbf{j}},\Omega}(\omega) = \frac{\delta(\mathbf{N}(\Phi,\omega)-\mathbf{j})}{Z_{\mathbf{j}}}$. Since $\Phi$ contains $\top$, we can clearly do this. Now, to represent a $\mathbb{C}$-MLN with an arbitrary given count distribution $q()$, we can just construct a convex combination of the $\mathbb{C}$-MLNs $\Phi_\mathbf{j}$. What remains to check is that a convex combination of $\mathbb{C}$-MLNs is still a $\mathbb{C}$-MLN. This is easy to see. Let $q(\mathbf{n}) = \sum_{\mathbf{j} \in \mathcal{J} \times \{1\}} A_\mathbf{j} \cdot \delta(\mathbf{n}-\mathbf{j})$. We construct a $\mathbb{C}$-MLN inducing this count distribution as follows. For all $\mathbf{j} \in \mathcal{J} \times \{1\}$, if $A_\mathbf{j} \neq 0$, we take the $\mathbb{C}$-MLN $\Phi_\mathbf{j} = \{(\alpha_1,\mathbf{w}_1^{(\mathbf{j})}),\dots, (\alpha_m,\mathbf{w}_m^{(\mathbf{j})}), (\top,\mathbf{w}_\top^{(\mathbf{j})}) \}$ (defined above) and modify it by changing the weight $\mathbf{w}_\top$ giving us the $\mathbb{C}$-MLN $\Phi_\mathbf{j}' = \{(\alpha_1,\mathbf{w}_1^{(\mathbf{j})}),\dots, (\alpha_m,\mathbf{w}_m^{(\mathbf{j})}), (\top,\mathbf{w}_\top^{(\mathbf{j})} + \ln(A_\mathbf{j}) - \ln(Z_\mathbf{j})) \}$ (recall that we ignore those $\mathbf{j}$'s for which $A_\mathbf{j}$ is zero and therefore there are no $\log 0$'s appearing anywhere). Now, it is not difficult to check that
\begin{multline*}
\widetilde{p}(\omega) = \sum_{\mathbf{j} \in \mathcal{J} \times \{1\}} p_{\Phi_\mathbf{j}',\Omega}(\omega) 
= \\ \sum_{\mathbf{j} \in \mathcal{J}\times \{1\}} \sum_{\mathbf{k} \in \mathcal{J} \times \{0\}} e^{ \langle i 2\pi \mathbf{k}/\mathbf{M}, \mathbf{N}(\Phi,\omega) \rangle } 
\cdot e^{- i 2 \pi \langle \mathbf{k}/\mathbf{M}, \mathbf{j} \rangle + \ln \frac{A_\mathbf{j}}{Z_\mathbf{j}} }.
\end{multline*}
is the probability given by a $\mathbb{C}$-MLN that we obtain by concatenating the respective weight vectors from the $\mathbb{C}$-MLNs $\Phi_\mathbf{j}'$ for all $\mathbf{j} \in \mathcal{J}$. Hence any count distribution is realizable by a $\mathbb{C}$-MLN, which finishes the proof.
\end{proof}

\begin{remark}
The inclusion of the tautology formula $\top$ in $\mathbb{C}$-MLNs is important for full expressivity.
\end{remark}

\begin{remark}
From the proof of Theorem \ref{thm:expressivity}, it may seem that we need $|\mathcal{J}|^2$ components in the $\mathbb{C}$-MLN in order to guarantee full expressivity. However, we can reduce the number of components of the $\mathbb{C}$-MLN by pushing the summation $\sum_{\mathbf{j} \in \mathcal{J}}$ inside and simplifying. 
We can then reduce the number of components of the $\mathbb{C}$-MLN to $|\mathcal{J}|$ (we omit the details here).\footnote{Alternatively, one could prove the result about expressivity of $\mathbb{C}$-MLNs using inverse DFT, which we discuss farther in the paper, however, we believe the present proof is more intuitive.}
\end{remark}


\subsection{WHY COMPLEX NUMBERS?}

At this point one could ask: Why complex numbers? Why not just specify the count distribution using a look-up table. After all, there are only polynomially (in $|\Delta$) many points in $\operatorname{Supp}(\Psi,\Omega)$ and typically exponentially many possible worlds in $\Omega$. The reason will become clear in the next section where we discuss inference in $\mathbb{C}$-MLNs. In short, {\em exact} inference in $\mathbb{C}$-MLNs can be done by the same algorithms (and circuits) as in classical MLNs; all we need to do is to replace real numbers by complex numbers.

\section{INFERENCE IN $\mathbb{C}$-MLNS}\label{sec:inference}

We can use WFOMC for inference in $\mathbb{C}$-MLNs (assuming that complex weights are supported), which turns out to be important for showing that we can use $\mathbb{C}$-MLNs for domain-lifted inference on distributions that are not expressible by normal MLNs.

In particular, to compute the partition function $Z$ of a $\mathbb{C}$-MLN $\Phi$, we proceed analogically to how one proceeds for normal MLNs (cf Section \ref{sec:wfomc}, which is based on existing works, e.g. \cite{van2011lifted}). Let a $\mathbb{C}$-MLN $\Phi = \{(\alpha_1,\mathbf{w}_1),\dots,(\alpha_m,\mathbf{w}_m) \}$ be given, with $d$ the dimension of the vectors $\mathbf{w}_i$. 
First, for every weighted formula $(\alpha_i,\mathbf{w}_i) \in \Phi$, where the free variables in $\alpha_i$ are exactly $x_1$, $\dots$, $x_k$, we create a new formula 
$$
    \forall x_1,\dots,x_k : \xi_i(x_1,\dots,x_k) \Leftrightarrow \alpha_i(x_1,\dots,x_k)
$$
where $\xi$ is a new fresh predicate. We denote the resulting set of new formulas $\Gamma$. 
Then we define $d$ different weight functions $w_1$, $w_2$, $\dots$, $w_d$ and we set
$w_j(\xi_i) = \exp{\left([\mathbf{w}_i]_j \right)}$
and $\overline{w}_j(\xi_i) = 1$ and for all other predicates we set all $w_1$, $\dots$, $w_d$ and $\overline{w}_1$, $\dots$, $\overline{w}_d$ equal to 1. It is easy to check that then 
\begin{equation}\label{eq:Z}
Z = \sum_{j=1}^d \mathbf{WFOMC}(\Gamma,w_j,\overline{w}_j),
\end{equation}
which is the partition function that we needed to compute. 
To compute the marginal probability of a given query $q$, we have 
$$\textit{Pr}_{\Phi,\Omega}[q] = \frac{\sum_{j=1}^d\mathbf{WFOMC}(\Gamma \cup \{ q \}, w_j, \overline{w}_j)}{\sum_{j=1}^d\mathbf{WFOMC}(\Gamma, w_j, \overline{w}_j)}.$$

\begin{remark}
The fact that we can compute the partition function and marginal probabilities of $\mathbb{C}$-MLNs using WFOMC is important. It means that whenever inference in an MLN using WFOMC is domain-liftable and if this extends to domain-liftability of WFOMC over $\mathbb{C}$, any $\mathbb{C}$-MLN with the same first-order logic formulas is domain-liftable as well. Importantly, the $\mathbb{C}$-MLN can model more distributions than its classical counterpart. 
\end{remark}

\subsection{NUMERICAL CONSIDERATIONS AND DOMAIN LIFTABILITY}\label{sec:numerics}

We need to pay enough attention to the representation of the complex numbers that appear during computations when performing inference with $\mathbb{C}$-MLNs.\footnote{Similar numerical issues for BNs were studied in \cite{dagum1993approximating}.} In particular, this is needed if we want to make claims about domain-lifted inference with complex weights by using existing algorithms, e.g.\ those based on FO-sda-DNNF circuits~\cite{van2013lifted}.

We assume only (i) rational numbers $\frac{a}{b}$ represented as a pair of integers $(a,b)$, (ii) complex numbers of the form $\frac{a}{b} \cdot e^{i 2\pi c/d}$ represented as a $4$-tuple of integers $(a,b,c,d)$ and (iii) complex numbers $\frac{a_1}{b_1} \cdot e^{i 2\pi c_1/d_1} + \dots + \frac{a_k}{b_k} \cdot e^{i 2\pi c_k/d_k}$ represented as a sequence of $4$-tuples $(a_1,b_1,c_1,d_1), \dots, (a_k,b_k,c_k,d_k)$. Note that any complex number can be arbitrarily well approximated by numbers of this form.

For evaluating FO-sda-DNNF circuits, one needs the following operations: (i) summation, (ii) multiplication, (iii) multiplication by $n' \in \mathbb{N}$ and (iv) exponentiation to the power of $n'' \in \mathbb{N}$ where $n' \leq n$ and $n'' \leq n$ and $n$ is 
polynomial in the size of the domain.

We start with a simple remark that follows from how we restricted the allowed representation of complex numbers.

\begin{remark}
Any number that can be produced from a finite set of numbers of the three types specified above using the four allowed operations can be represented using $O(1)$ $4$-tuples $(a_1,b_1,c_1,d_1), \dots, (a_k,b_k,c_k,d_k)$ of rational numbers representing the sum $\frac{a_1}{b_1} \cdot e^{i 2\pi c_1/d_1} + \dots + \frac{a_k}{b_k} \cdot e^{i 2\pi c_k/d_k}$.\footnote{The big-$O$ notation used here is w.r.t.\ the domain size $|\Delta|$.} This follows from the fact that the arguments of the complex exponentials can only be fractional multiples of $2\pi$ here.
\end{remark}

In principle, the above remark allows us to obtain an upper bound on the representation size needed for a single complex number during computation of WFOMC. What remains is to bound the size of the representation of the integers $a_i,b_i,c_i,d_i$. For that one needs to usually dig into details of the specific algorithms inference. For instance, when computing WFOMC by evaluating an FO-sda-DNNF circuit, there is a constant number of exponentiations on any path from the root of the circuit to its leaves. There is also always just a polynomial number (in the domain size) of applications of the four numerical operations in the circuit. One can then show that the representation of the complex numbers needed during the computation is polynomial in the domain size. This is enough to show that domain-liftability using this approach transfers to complex weights as well. We omit the tedious details here.

\section{DFT OF COUNT DISTRIBUTIONS}

In this section we look at the DFT of the count distributions $q_{\Phi,\Omega}$ induced by an MLN (in this section it will still be the classical MLN without complex weights to simplify the exposition). 
We show that to compute the DFT, all we need is to be able to compute several WFOMCs with complex weights. 

In what follows in this section, let $\Phi = \{ (\alpha_1, w_1), \dots, (\alpha_m,w_m) \}$ be an MLN and $\Omega$ be a set of possible worlds on a domain $\Delta$. We also define $\Psi = \{\alpha_1,\dots, \alpha_m \}$ to be the respective set of first-order logic formulas contained in $\Phi$ and $\mathbf{w} = [w_1,w_2,\dots,w_m]$ to be the respective vector of weights from the MLN.

We want to compute the DFT of $q_{\Phi,\Omega}(\mathbf{n})$. Here, $q_{\Phi,\Omega}(\mathbf{n})$ is a real-valued function of $m$-dimensional integer vectors. We can restrict the domain\footnote{Here, {\em domain} refers to the domain of a mathematical function, not to a {\em domain} as a set of domain elements.} of $q_{\Phi,\Omega}(\mathbf{n})$ to the set $\mathcal{D} = \left\{0,1,\dots, |\Delta|^{|\textit{vars}(\alpha_1)|} \right\} \times \dots \times \left\{ 0, 1, \dots, |\Delta|^{|\textit{vars}(\alpha_m)|} \right\}$. 
This still ensures that it will always be the case that $\operatorname{Supp}(\Phi,\Omega) \subseteq \mathcal{D}$.

From the definition of DFT we have
\begin{equation}\label{dftq1}
    g_{\Phi,\Omega}(\mathbf{k}) = \mathcal{F} \left\{ q_{\Phi,\Omega} \right\} = \sum_{\mathbf{n} \in \mathcal{D}} q_{\Phi,\Omega}(\mathbf{n})  e^{ - i 2 \pi \langle \mathbf{k}, \mathbf{n} / \mathbf{M} \rangle }
\end{equation}
where $\mathbf{M} = \left[|\Delta|^{|\textit{vars}(\alpha_1)|}+1,\dots,|\Delta|^{|\textit{vars}(\alpha_m)|}+1\right]$ and the division in $\mathbf{n} / \mathbf{M}$ is again component-wise.

Plugging in the definition of $q_{\Phi,\Omega}(\mathbf{n})$ into (\ref{dftq1}), we obtain
\begin{multline*}
    g_{\Phi,\Omega}(\mathbf{k})= \sum_{\mathbf{n} \in \mathcal{D}}\sum_{\omega \in \Omega : \mathbf{N}(\Phi,\omega) = \mathbf{n}} p_{\Phi,\Omega}(\omega)  e^{ - i 2 \pi \langle \mathbf{k}, \mathbf{n} / \mathbf{M} \rangle } \\
    = \sum_{\mathbf{n} \in \mathcal{D}}\sum_{\omega \in \Omega : \mathbf{N}(\Phi,\omega) = \mathbf{n}} \frac{1}{Z} e^{\langle \mathbf{w}, \mathbf{N}(\Phi,\omega) \rangle}  e^{ - i 2 \pi \langle \mathbf{k}, \mathbf{n} / \mathbf{M} \rangle } \\
    = \sum_{\mathbf{n} \in \mathcal{D}}\sum_{\omega \in \Omega : \mathbf{N}(\Phi,\omega) = \mathbf{n}} \frac{1}{Z} e^{\langle \mathbf{w}, \mathbf{N}(\Phi,\omega) \rangle}  e^{ - i 2 \pi \langle \mathbf{k}/\mathbf{M}, \mathbf{n} \rangle } \\
    = \sum_{\mathbf{n} \in \mathcal{D}}\sum_{\omega \in \Omega : \mathbf{N}(\Phi,\omega) = \mathbf{n}} \frac{1}{Z} e^{\langle \mathbf{w}, \mathbf{N}(\Phi,\omega) \rangle}  e^{ - i 2 \pi \langle \mathbf{k}/\mathbf{M}, \mathbf{N}(\Phi,\omega) \rangle } \\
    = \frac{1}{Z} \sum_{\omega \in \Omega}  e^{\langle \mathbf{w}- i 2 \pi \mathbf{k}/\mathbf{M}, \mathbf{N}(\Phi,\omega) \rangle}.
\end{multline*}
Note that, in the above, $Z$ is the partition function of the original MLN $\Phi$. For simplicity of exposition, $\Phi$ was not a $\mathbb{C}$-MLN here but the same approach also works for $\mathbb{C}$-MLNs (with one difference being that there we also need to sum over all  components of the $\mathbb{C}$-MLN and hence we need more calls to the WFOMC oracle).



\subsection{COMPUTING DFT USING WFOMC}\label{sec:dftusingwfomc}

In the previous section, we found the form of the DFT of the count distribution of an MLN $\Phi$ to be

$$g_{\Phi,\Omega}(\mathbf{k}) = \frac{1}{Z} \sum_{\omega \in \Omega}  e^{\langle \mathbf{w}- i 2 \pi \mathbf{k}/\mathbf{M}, \mathbf{N}(\Phi,\omega) \rangle}. $$

Here we show how to compute $g_{\Phi,\Omega}(\mathbf{k})$ using WFOMC. First, computing $Z$ is simple. It is just the partition function $Z$ of the MLN $\Phi$. Therefore we can compute it using~(\ref{eq:Z}). 

Next, to compute the sum $\sum_{\omega \in \Omega}  e^{\langle \mathbf{w}- i 2 \pi \mathbf{k}/\mathbf{M}, \mathbf{N}(\Phi,\omega) \rangle}$, we can notice that it is again a partition function, but of another {\em complex} MLN. It is the partition function of the $\mathbb{C}$-MLN 
$\Phi_{\mathbf{k}} = \{ (\alpha_1,[\mathbf{w}- i 2 \pi \mathbf{k}/\mathbf{M}]_1), \dots, (\alpha_m,[\mathbf{w}- i 2 \pi \mathbf{k}/\mathbf{M}]_m) \}$ (we recall that $[\mathbf{v}]_j$ is used to denote the $j$-th entry of the vector $\mathbf{v}$). We can therefore use (\ref{eq:Z}) to compute the sum using WFOMC. 

\subsection{CONSTRUCTING MARGINAL POLYTOPES}

Here we use the techniques developed in this paper to design algorithms for computing count distributions of MLNs and for constructing so called {\em relational marginal polytopes}. The algorithm for constructing relational marginal polytopes will turn out to need a much smaller number of WFOMC oracle calls than a recently published domain-lifted algorithm from \cite{kuzelka.aistats.2020}. 

Let us first recall the definition of {\em relational marginal polytopes} \cite{kuzelka2018relational,kuzelka.aistats.2020}.

\begin{definition}[Relational marginal polytope]
Let $\Omega$ be the set of possible worlds on domain $\Delta$ and $\Psi = (\alpha_1, \dots, \alpha_m)$ be a list of formulas. First we define $Q(\alpha,\omega) := \frac{N(\alpha,\omega)}{|\Delta|^{|\textit{vars}(\alpha)|}}$. Then we define the relational marginal polytope $\mathbf{RMP}(\Psi,\Delta)$ w.r.t.\ $\Psi$ 
as $\mathbf{RMP}(\Psi,\Delta) = \{ (x_1,\dots,x_m) \in \mathbb{R}^m : \exists \mbox{ dist.\ on } 
     \Omega\; { s.t. } \;\mathbb{E}[Q(\alpha_1,\omega)] = x_1 \wedge \dots \wedge \mathbb{E}[Q(\alpha_m,\omega)] = x_m \}.$
\end{definition}

Note that, for $\Psi = (\alpha_1, \dots, \alpha_m)$, the relational marginal polytope $\mathbf{RMP}(\Psi,\Delta)$ is just the convex hull of the set of points 
$\left\{ \left(n_1/|\Delta|^{|\textit{vars}(\alpha_1)|}, \dots, n_m/|\Delta|^{|\textit{vars}(\alpha_m)|}\right) \right.$ $\left. \left|  (n_1,\dots,n_m) \in \operatorname{Supp}(\Phi,\Omega) \right. \right\}$. 
where $\Omega$ is the set of all possible worlds over the domain~$\Delta$.

Therefore all we need to do is to compute the set $\operatorname{Supp}(\Psi,\Omega)$ and the rest is just a standard geometric problem of finding the convex hull of a set of points for which one can employ existing algorithms (e.g.\ \cite{barber1996quickhull}).

First, we notice that we can extract the count distribution $q_{\Phi,\Omega}$ of any MLN $\Phi$ using just WFOMC. In fact, with the material developed so far in this paper, this is an extremely easy thing to do. We just compute the DFT of the $q_{\Phi,\Omega}$ using WFOMC as in Section~\ref{sec:dftusingwfomc}, then convert the result back using inverse DFT and the result is the count distribution $q_{\Phi,\Omega}$ that we wanted to compute. To compute the set $\operatorname{Supp}(\Psi,\Omega)$ for a set $\Psi = \{\alpha_1,\dots,\alpha_m \}$, we can just compute the count distribution of the MLN $\Phi = \{(\alpha_1,0),\dots,(\alpha_m,0) \}$, which defines a uniform distribution over possible worlds, and select the points with non-zero probability (there are only polynomially many such points in the size of the domain $\Delta$). After that, we are done.

As a direct consequence of the above, we have the next theorem.

\begin{theorem}\label{thm:thm2}
Let $\Psi = \{\alpha_1,\dots,\alpha_m \}$ be a set of first-order logic formulas, $\Phi = \{(\alpha_1,w_1),\dots,(\alpha_m,w_m) \}$ be an MLN, $\Delta$ be a set of domain elements and $\Omega$ be the set of all possible worlds over the domain $\Delta$. Given a WFOMC oracle, both the relational marginal polytope $\mathbf{RMP}(\Psi,\Delta)$ and the count distribution $q_{\Phi,\Omega}(\mathbf{n})$ can be constructed in time polynomial in $|\Delta|$ using $|\Delta|^{\sum_{i=1}^m |\textit{vars}(\alpha_i)|}+1$ calls to the WFOMC oracle.
\end{theorem}

\begin{remark}
The algorithm from \cite{kuzelka.aistats.2020} needs $|\Delta|^{|\Psi| \cdot \sum_{i=1}^m |\textit{vars}(\alpha_i)|}$ calls to a WFOMC oracle. The presented method is therefore a substantial improvement.
\end{remark}

\begin{remark}
Naively, one would expect that computing the count distribution of an MLN requires to query the MLN using cardinality constraints (which are not supported by existing WFOMC systems) such as: ``compute the probability that $N(\alpha_1,\omega) = n_1$, $\dots$, $N(\alpha_m,\omega) = n_m$ holds in a possible world $\omega$ sampled from the MLN.'' So it might be a bit surprising at first that we can compute the count distribution using just WFOMC without any cardinality constraints  (albeit with complex weights). 
\end{remark}

\subsubsection{AN EXAMPLE}

Here we show a concrete example of counting distributions of an MLN. We use the ``friends and smokers'' MLN~\cite{Richardson2006}, which is given as $\Phi = \{(\textit{sm}(x), w_1), (\textit{sm}(x) \wedge \textit{fr}(x,y) \Rightarrow \textit{sm}(y), w_2) \}$. We set $w_1 = 0$ and $w_2 = 0$ and $\Delta = \{A_1,A_2,\dots,A_{10}\}$. This MLN models a uniform distribution over possible worlds. The resulting count distribution computed by the algorithm outlined in this section is shown in Figure \ref{fig:count_distribution_example}. While the distribution on possible worlds $p_{\Phi,\Omega}(\omega)$ is uniform, the corresponding count distribution $q_{\Phi,\Omega}(\mathbf{n})$ is obviously not.

\begin{figure}[t]
\includegraphics[width=1\linewidth]{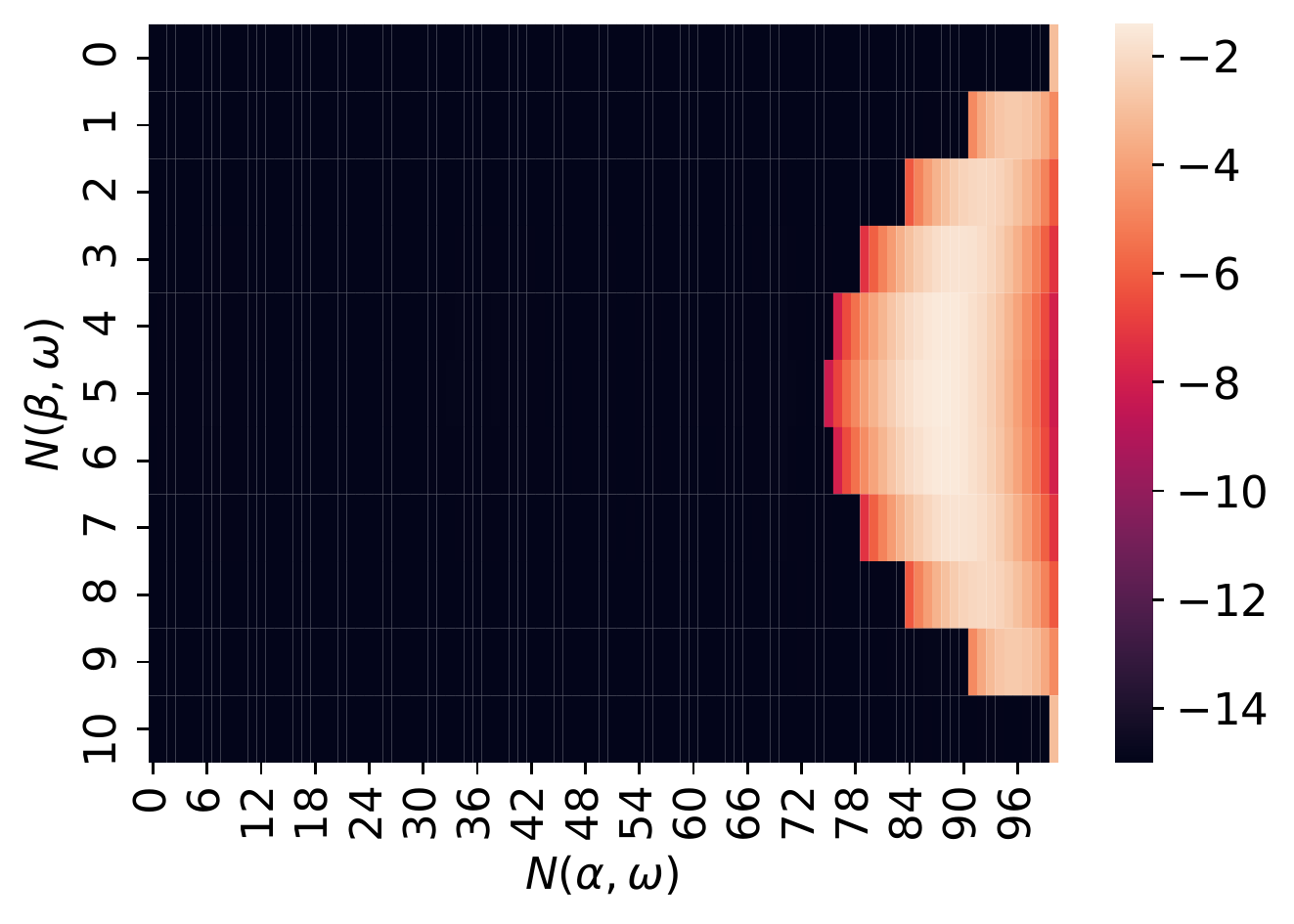}
\caption{The count distribution of the MLN $\Phi = \{(\textit{sm}(x), 0), (\textit{sm}(x) \wedge \textit{fr}(x,y) \Rightarrow \textit{sm}(y), 0) \}$. The x-axis corresponds to $N(\textit{sm}(x) \wedge \textit{fr}(x,y) \Rightarrow \textit{sm}(y),\omega)$ and the y-axis to $N(\textit{sm}(x),\omega)$. Shown in logarithmic scale.}\label{fig:count_distribution_example}
\end{figure}



\section{RELATED WORK}

We are not the first to study complex numbers in the context of statistical relational learning. Buchman and Poole \cite{DBLP:conf/uai/BuchmanP17} extended probabilistic logic programs (PLPs) with complex weights and showed that this leads to full expressivity. However, their notion of expressivity is different from the one we use here. Their motivation comes from an observation from their earlier work \cite{DBLP:conf/kr/BuchmanP16} where they showed that there are distributions that cannot be represented by any {\em relational} PLP without constants no matter how complicated the probabilistic rules defining it are. Their positive result from \cite{DBLP:conf/uai/BuchmanP17} uses PLPs in a special form and does not apply to probabilistic logic programs with a fixed set of rules. Besides the fact that we study MLNs and not PLPs, the main difference is that we focus on expressivity of fixed (complex) MLNs. Importantly, this allows us to broaden the class of distributions that can be modelled {\em efficiently}.

Another closely related contribution from the statistical relational learning literature is the work of Van den Broeck, Meert and Darwiche \cite{van2014skolemization} who used negative weights in order to ``skolemize'' first-order logic sentences with existential quantifiers, which was later generalized in \cite{meert2016relaxed}. This work is orthogonal to ours; we can use the Skolemization procedure to get rid of existential quantifiers when performing inference using WFOMC in a $\mathbb{C}$-MLN containing existential quantifiers.

\section{CONCLUSIONS}

We started by studying the expressivity of Markov logic networks (MLNs). After observing that MLNs are not (almost) fully expressive, we introduced {\em complex MLNs}. We then showed that, unlike their standard counterparts, complex MLNs are {\em fully expressive}. This has an important consequence for tractable probabilistic modelling. It means that we can reason about larger class of distributions in time polynomial in the domain size than what was previously known, but, importantly, using already existing lifted inference algorithms! 

Finally, after noticing that discrete Fourier transform of a count distribution can be computed using WFOMC with complex weights, we proposed an algorithm for computing count distributions and relational marginal polytopes using WFOMC oracles. We believe that being able to compute count distributions of MLNs efficiently will be useful for modelling population, ``class-level'' \cite{schulte}, statistics using MLNs.

\paragraph{Acknowledgements} This work was supported by Czech Science Foundation project ``Generative Relational Models'' (20-19104Y), the OP VVV project {\it CZ.02.1.01/0.0/0.0/16\_019/0000765} ``Research Center for Informatics'' and a donation from X-Order Lab.


\bibliographystyle{plain}
\bibliography{ijcai20}

\end{document}